\documentclass[titlepage,12pt]{article} 
\usepackage{amsmath}
\usepackage{amssymb}
\usepackage{amsthm}
\usepackage{setspace}
\usepackage[nottoc]{tocbibind}
\usepackage{graphicx}
\usepackage[all]{xy}
\usepackage{enumerate}
\usepackage[margin=3.2cm]{geometry}
\newcommand{\unit}{\left[0,1\right]}

\begin{document}
\newtheorem{them}{Theorem}[section]
\newtheorem{cor}[them]{Corollary}
\newtheorem{defin}[them]{Definition}
\newtheorem{prop}[them]{Proposition}
\newtheorem{example}[them]{Example}
\newtheorem{lemma}[them]{Lemma}
\newtheorem{conj}[them]{Conjecture}
\newtheorem{corollary}[them]{Corollary}

\begin{titlepage}
\null\vfill
\begin{center}
\LARGE
Bounding the Fat Shattering Dimension of a Composition Function Class Built Using a Continuous Logic Connective
\par
\vskip3cm
\normalsize
By\par
Hubert Haoyang Duan
\par
\vskip1cm
Supervised By
\par
Dr. Vladimir Pestov
\par
\vskip3cm
Project submitted in partial fulfillment of the requirements for the degree of B.Sc. Honours Specialization in Mathematics
\par
\vskip3cm
Department of Mathematics and Statistics\par
Faculty of Science\par
University of Ottawa
\par
\vskip2cm
\copyright Hubert Haoyang Duan, Ottawa, Canada, 2011
\end{center}
\vfill
\end{titlepage}
\normalsize
\singlespacing
\begin{abstract}
\normalsize
We begin this report by describing the Probably Approximately Correct (PAC) model for
learning a concept class, consisting of subsets of a domain, and a
function class, consisting of functions from the domain to the unit
interval. Two combinatorial parameters, the Vapnik-Chervonenkis (VC)
dimension and its generalization, the Fat Shattering dimension of scale $\epsilon$,
are explained and a few examples of their calculations are given with proofs. We then explain Sauer's
Lemma, which involves the VC dimension and is used to prove the equivalence of a
concept class being distribution-free PAC learnable and it having finite
VC dimension.

As the main new result of our research, we explore the construction of a
new function class, obtained by forming compositions with a continuous logic connective, a
uniformly continuous function from the unit hypercube to the unit
interval, from a collection of function classes. Vidyasagar had proved that such a composition function class has finite Fat Shattering dimension of all scales if the classes in the original collection do; however, no estimates of the dimension were known. Using results by Mendelson-Vershynin and Talagrand, we bound the Fat Shattering dimension of scale $\epsilon$ of this new function class in terms of the
Fat Shattering dimensions of the collection's classes.

We conclude this report by providing a few open questions and future research topics involving the PAC learning model.
\end{abstract}
\newpage

\setcounter{tocdepth}{2}
\tableofcontents

\newpage
\normalsize
\section{Introduction}
In the area of statistical learning theory, the Probably Approximately Correct (PAC) learning model formalizes the notion of learning by using sample data points to produce valid hypotheses through algorithms. For instance, the following illustrates one learning problem which can be formalized in the PAC model. Given that there is a disease which affects certain people and out of 100 people in a hospital, 12 of them are sick with this disease. Is there a way to predict whether any given person in the hospital has the disease or not?

This report covers the PAC learning model applied to learning a collection of subsets $\mathcal{C}$, called a concept class, of a domain $X$ and more generally, a collection of functions $\mathcal{F}$, called a function class, from $X$ to the unit interval $\unit$. The report involves mostly concepts from analysis and some concepts from probability theory, but only the completion of the first two years of undergraduate studies in mathematics are assumed from the readers.

\subsubsection*{Report outline}

First, we give two definitions of PAC learning, one for a concept class $\mathcal{C}$ and the other for a function class $\mathcal{F}$, and explore two combinatorial parameters, the Vapnik-Chervonenkis (VC) dimension and the Fat Shattering dimension of scale $\epsilon$, for $\mathcal{C}$ and $\mathcal{F}$, respectively. 
Then, we explain Sauer's Lemma, a theorem which involves the VC dimension of $\mathcal{C}$ and is used to prove that the finiteness of this dimension is a sufficient condition for $\mathcal{C}$ to be learnable.

Finally, as the main new result of our research, given function classes $\mathcal{F}_1,\ldots,\mathcal{F}_k$ and a ``continuous logic connective" (that is, a continuous function $u:\unit^k\to\unit$), we consider the construction of a new composition function class $u(\mathcal{F}_1,\ldots,\mathcal{F}_k)$, consisting of functions $u(f_1,\ldots,f_k)$ defined by $u(f_1,\ldots,f_k)(x) = u(f_1(x),\ldots,f_k(x))$ for $f_i\in\mathcal{F}_i$. We then bound the Fat Shattering dimension of scale $\epsilon$ of this class in terms of a sum of the Fat Shattering dimensions of scale $\delta(\epsilon,k)$ of $\mathcal{F}_1,\ldots,\mathcal{F}_k$, where $\delta(\epsilon,k)$ only depends on $\epsilon$ and $k$. There is a previously known analogous estimate for a composition of concept classes built using a usual connective of classical logic \cite{vid}. We deduce our new bound using results from Mendelson-Vershynin and Talagrand.

Before jumping into the PAC learning model, we provide some basic terminology and results from analysis and measure theory. From now on, any propositions or examples given with proofs, unless mentioned otherwise, are done by us and are independent of any sources.
\newpage
\section{Brief Overview of Analysis and Measure Theory}

This section lists some definitions and results in measure theory and analysis, found in standard textbooks, such as \cite{measuretheory}, \cite{vid}, and \cite{french}, which are used in this report. 

\subsection*{Probability space}

\begin{defin}
Let $X$ be a set. A {\em $\sigma$-algebra} $\mathcal{S}$ is a non-empty collection of subsets of $X$ such that the following are satisfied:
\begin{enumerate}
\item If $A\in \mathcal{S}$, then $X\setminus A\in \mathcal{S}$
\item If $A_i\in \mathcal{S}$ for $i\in\mathbb{N}$, then $\displaystyle\bigcup_{i\in\mathbb{N}} A_i\in \mathcal{S}$
\end{enumerate}
If $\mathcal{S}$ is a $\sigma$-algebra, then the pair $(X,\mathcal{S})$ is called a {\em measurable space}.
\end{defin}

\begin{defin}
Suppose $(X,\mathcal{S})$ and $(Y,\mathcal{T})$ are two measurable spaces. A function $f:X\to Y$ is called {\em measurable} if $f^{-1}(T)\in \mathcal{S}$ for all $T\in \mathcal{T}$.
\end{defin}
\begin{defin}
Given a measurable space $(X,\mathcal{S})$, a function $\mu:\mathcal{S}\to\mathbb{R}^+=\{r\in\mathbb{R}:r\geq 0\}$ is a {\em measure} if the following hold:
\begin{enumerate}
\item $\mu(\emptyset)=0$
\item If $A_i\in\mathcal{S}$ for all $i\in\mathbb{N}$ and $A_i\cap A_j=\emptyset$ whenever $i\neq j$, then
\[
\mu \left(\bigcup_{i\in\mathbb{N}} A_i\right)=\sum_{i\in\mathbb{N}} \mu(A_i)
\]
\end{enumerate}
The triple $(X,\mathcal{S},\mu)$ is called a {\em measure space}. If in addition, $\mu$ satisfies $\mu(X) = 1$, then $\mu$ is a {\em probability measure} and $(X,\mathcal{S},\mu)$ is called a {\em probability space}. 
\end{defin}

Given a probability space $(X,\mathcal{S},\mu)$, one can measure the difference between two subsets $A,B\in\mathcal{S}$ of $X$ by looking at their symmetric difference $A\bigtriangleup B$, which is indeed in $\mathcal{S}$:
\begin{align*}
\mu(A \bigtriangleup B) &= \mu( (A\cup B) \setminus (A\cap B))\\
& = \mu ( ((X\setminus A) \cap B)\cup (A\cap (X\setminus B))).
\end{align*}
More generally, given two measurable functions $f,g:X\to\unit$, one can look at the expected value of their absolute difference by integrating with respect to $\mu$:
\[
\int_X |f(x) - g(x)|\, d\mu(x).
\]
This report does not go into any details involving the Lebesgue integral but does assume that integration of measurable functions to the real numbers, which is a measure space, makes sense and is linear and order-preserving:
\[
\int_X (rf(x) + r'g(x)) \, d\mu(x) = r\int_X f(x)\,d\mu(x) + r'\int_X g(x)\,d\mu(x)
\]
and
\[
\int_X f(x) \,d\mu(x) \leq \int_X g(x)\,d\mu(x),
\]
if $f(x)\leq g(x)$ for all $x\in X$. 

Validating hypotheses in the PAC learning model uses the idea of measuring the symmetric difference of two subsets of a probability space $(X,\mathcal{S},\mu)$ and calculating the expected value of the difference of $f,g:X\to\unit$. The structure of metric spaces arises naturally from these two notions.

\subsection*{Metric spaces}

\begin{defin}
Let $M$ be a nonempty set. A function $d:M\times M\to \mathbb{R}^+$ is a {\em metric} if the following hold for all $m_1,m_2,m_3\in M$:
\begin{enumerate}
\item $d(m_1,m_2)=0$ if and only if $m_1=m_2$
\item $d(m_1,m_2)=d(m_2,m_1)$
\item $d(m_1,m_2)\leq d(m_1,m_3)+d(m_3,m_2)$
\end{enumerate}
In this case, the pair $(M,d)$ is called a {\em metric space}.
\end{defin}

\begin{defin}
Given a metric space $(M,d)$, a {\em metric sub-space} of $M$ (which is a metric space in its own right) is a nonempty subset $M'\subseteq M$ equipped with the distance $d_{|_{M'}}$, the restriction of $d$ to $M'$.
\end{defin}





The structure of a metric space exists in every vector space equipped with a norm.


\begin{defin}\label{vectorspacenorm}
Suppose $V$ is a vector space over $\mathbb{R}$. A function $\rho: V \to \mathbb{R}^+$ is a {\em norm} on $V$ if for all $v_1,v_2\in V$ and for all $r\in \mathbb{R}$,
\begin{enumerate}
\item $\rho(rv_1 ) = |r| \rho (v_1)$
\item $\rho( v_1 + v_2) \leq \rho( v_1) + \rho( v_2)$
\item $\rho(v_1) = 0$ if and only if $v_1 = 0$
\end{enumerate}
If $\rho$ is a norm on $V$, then $(V,\rho)$ is called a {\em normed vector space}.
\end{defin}

\begin{prop}
Based on Definition \ref{vectorspacenorm}, the function $d:V\times V\to\mathbb{R}^+$ defined by $d(u,v) = \rho(u - v)$ is a metric on $V$, and $d$ is called the {\em metric induced by the norm} $\rho$ on $V$.
\end{prop}

The following subsection provides a few examples of metric spaces which will be encountered in this report.

\subsubsection*{Examples of metric spaces}\label{appendixmetric}

The real numbers $(\mathbb{R},\rho)$, with the absolute value norm $\rho(r) = |r|$ for $r\in\mathbb{R}$, is a normed vector space so $\mathbb{R}$ can be equipped with a metric structure.
\begin{example}\label{examplerealsmetric}
The set $\mathbb{R}$ with distance $d$ defined by $d(r_1,r_2)= |r_1-r_2|$ for $r_1,r_2\in\mathbb{R}$ is a metric space.
\end{example}

The unit interval $\unit$ is a subset of $\mathbb{R}$, so it is a metric sub-space of $(\mathbb{R},d)$, and this space will be used quite often in this report.

Given a probability space $(X,\mathcal{S},\mu)$, the set $V$ of all bounded measurable functions from $X$ to $\mathbb{R}$ is a vector space, with point-wise addition and scalar multiplication. The function $\rho: V \to \mathbb{R}^ + $ defined by
\[
\rho (f) = \sqrt{\left(\int_X (f(x))^2 d\mu(x)\right)}
\]
is a norm on $V$ if any two functions $f,g: X \to \mathbb{R}$ which agree on a subset of $X$ with full measure, $\mu(\{x\in X: f(x)= g(x)\})= 1$, are identified.\footnote{This identification can be done using an equivalence relation, so this report will not go into any details here.} The norm $\rho$ is called the {\em $L_2(\mu)$ norm} on $V$ and we normally write $||f||_2 = \rho(f)$ for $f\in V$. As a result, $V$ can be turned into a metric space.

\begin{example}\label{l2normexample}

Following the notations in the paragraph above, $V$ is a metric space with distance $d$ defined by
\[
d(f,g) = ||f-g||_2 = \sqrt{\left(\int_X (f(x)-g(x))^2 d\mu(x)\right)}.
\]
\end{example}

Write $\unit^X$ for the set of all measurable functions from a probability space $(X,\mathcal{S},\mu)$ to $\unit$. Then, it is a metric sub-space of $V$ with distance induced by the $L_2(\mu)$ norm on $V$, restricted of course to $\unit^X$.

Given metric spaces $(M_1,d_1),\ldots, (M_k,d_k)$, their product $M_1\times\ldots\times M_k$ always has a metric structure.

\begin{example}\label{productmetric}
If $(M_1,d_1),\ldots, (M_k,d_k)$ are metric spaces, then their product $M_1\times\ldots\times M_k$ is a metric space with distance $d^2$ defined by
\[
d^2((m_1,\ldots,m_k),(m'_1,\ldots,m'_k)) = \sqrt{\left((d_1(m_1,m_1'))^2 +\ldots + (d_k(m_k,m_k'))^2\right)}.
\]
The distance $d^2$ is normally referred to as the {\em $L_2$ product distance} on $M_1\times\ldots\times M_k$.
\end{example}

From Examples \ref{examplerealsmetric} and \ref{productmetric}, the set $\unit^k$, which denotes the set-theoretic product $\unit\times\ldots\times\unit$ is then a metric space with distance $d^2$ defined by
\[
d^2((r_1,\ldots,r_k),(r_1',\ldots,r_k')) = \sqrt{\left( |r_1 - r_1'|^2 + \ldots + |r_k - r_k'|^2\right)}.
\]
Also, following Examples \ref{l2normexample} and \ref{productmetric}, if $\mathcal{F}_1,\ldots,\mathcal{F}_k$ are sets of measurable functions from a probability space $(X,\mathcal{S},\mu)$ to the unit interval, then $\mathcal{F}_i\subseteq \unit^X$ for each $i=1,\ldots, k$. Therefore, the product $\mathcal{F}_1\times\ldots\times\mathcal{F}_k$ is a metric space with distance defined by
\[
d^2((f_1,\ldots,f_k),(f'_1,\ldots,f'_k)) = \sqrt{\left((||f_1 - f'_1||_2)^2 + \ldots + (||f_k - f'_k||_2)^2\right)}.
\]


\newpage

\section{The Probably Approximately Correct Learning Model}\label{pacsection}

Let $(X,\mathcal{S})$ be a measurable space. A {\em concept class} $\mathcal{C}$ of $X$ is a subset of $\mathcal{S}$ and an element $A\in\mathcal{C}$ (a measurable subset of $X$) is called a {\em concept}. A {\em function class} $\mathcal{F}$ is a collection of measurable functions from $X$ to the unit interval $\unit$. Unless stated otherwise, from this section onwards, the following notations will be used:
\begin{enumerate}
\item $X = (X,\mathcal{S})$: a {\em measurable space}
\item $\mu$: a {\em probability measure} $\mathcal{S}\to \mathbb{R}^+$
\item $\mathcal{C}$: a {\em concept class} and $\mathcal{F}$: a {\em function class}
\item $\unit^X$: the set of all {\em measurable} functions $f:X\to\unit$, instead of the customary notation of all functions from $X$ to $\unit$.
\end{enumerate}

This section provides the definitions of learning $\mathcal{C}$ and $\mathcal{F}$ in the Probably Approximately Correct (PAC) learning model, introduced in 1984 by Valiant. 

Concept class PAC learning involves producing a valid hypothesis for every concept $A\in\mathcal{C}$ by first drawing random points, forming a training sample, from $X$ labeled with whether these points are contained in $A$. In other words, a labeled sample of $m$ points $x_1,\ldots,x_m\in X$ for $A$ consists of these points and the evaluations $\chi_A(x_1),\ldots,\chi_A(x_m)$ of the indicator function $\chi_A:X\to\{0,1\}$, where
\[
\chi_A(x) = 1 \textnormal{ if and only if }x\in A. 
\]
On the other hand, an unlabeled sample of points does not include these evaluations. The set of all labeled samples of $m$ points can then be identified with $(X\times\{0,1\})^m$, and producing a hypothesis for $A$ with a labeled sample is exactly the process of associating the sample to a concept $H\in \mathcal{C}$ (i.e. this process is a function from the set of all labeled samples to the concept class).

Here is the precise definition of a concept class being learnable.

\begin{defin}[\cite{pacdefinition}]\label{paclearningconcept}
A concept class $\mathcal{C}$ is {\em distribution-free Probably Approximately Correct learnable} if there exists an algorithm\footnote{In this report, a learning algorithm is simply defined to be a function.} $L:\cup_{m\in\mathbb{N}}(X\times\{0,1\})^m\to\mathcal{C}$ with the following property: for every $\epsilon>0$, for every $\delta>0$, there exists a $M\in\mathbb{N}$ such that for every $A\in\mathcal{C}$, for every probability measure $\mu$, for every $m\geq M$, for any $x_1,\ldots,x_m\in X$, we have $\mu(H_m\bigtriangleup A)<\epsilon$ with confidence at least $1-\delta$, where $H_m=L((x_1,\chi_A(x_1)),\ldots,(x_m,\chi_A(x_m)))$.
\end{defin}

Confidence of at least $1-\delta$ in the definition above, keeping to the same notations, simply means that the (product) measure of the set of all $m$-tuples $(x_1,\ldots,x_m)\in X^m$, where $\mu(H_m\bigtriangleup A)<\epsilon$ for $H_m=L((x_1,\chi_A(x_1)),\ldots,(x_m,\chi_A(x_m)))$, is at least $1-\delta$. In other words, an equivalent statement to $\mathcal{C}$ is distribution-free PAC learnable is that for every $\epsilon,\delta>0$, there exists $M\in\mathbb{N}$ such that for every $A\in\mathcal{C}$, probability measure $\mu$, and $m\geq M$,
\[
\mu^m(\{(x_1,\ldots,x_m)\in X^m :\mu(H_m\bigtriangleup A)\geq \epsilon\})\leq \delta,\footnote{The symbol $\mu^m$ denotes the product measure on $X^m$; the reader can refer to \cite{measuretheory} for the details.}
\]
for $H_m=L((x_1,\chi_A(x_1)),\ldots,(x_m,\chi_A(x_m)))$.

A concept class $\mathcal{C}$ is distribution-free learnable in the PAC learning model if a hypothesis $H$ can always be constructed from an algorithm $L$ for every concept $A\in \mathcal{C}$, using any labeled sample for $A$, such that the measure of their symmetric difference $H\bigtriangleup A$ is arbitrarily small with respect to every probability measure and with arbitrarily high confidence, as long as the sample size is large enough.

Every concept $A\in\mathcal{C}$ is a subset of $X$ so $A$ can be associated to its indicator function $\chi_A: X\to\{0,1\}$. Even more generally, $\chi_A$ is a function from $X$  to $\unit$; in other words, every concept class $\mathcal{C}$ can be identified as a function class $\mathcal{F}_\mathcal{C}=\{\chi_A:X\to\unit:A\in\mathcal{C}\}$, so it is natural to generalize Definition \ref{paclearningconcept} for any function class $\mathcal{F}$.

Definition \ref{paclearningconcept} involves the symmetric difference of two concepts and its generalization to measurable functions $f,g:X\to\unit$ is the expected value of their absolute difference $E_\mu(f,g)$, as seen in the previous section:
\[
E_\mu(f,g) = \int_X |f(x) - g(x)|\, d\mu(x).
\]
A simple exercise can show that if $f,g\in\unit^X$ take values in $\{0,1\}$, so they are indicator functions of two concepts $A,B\subseteq X$, then $E_\mu(f,g)$ coincide with the measure of their symmetric difference: $E_\mu(f,g)=\mu(A\bigtriangleup B)$, where $f=\chi_A$ and $g=\chi_B$.

With the generalization of the symmetric difference, distribution-free PAC learning for any function class can be defined. In the context of function class learning, a labeled sample of $m$ points $x_1,\ldots,x_m\in X$ for a function $f\in\mathcal{F}$ consists of these points and the evaluations $f(x_1),\ldots,f(x_m)$. Then, the set of all labeled samples of $m$ points can be identified with $(X\times\unit)^m$, and producing a hypothesis is the process of associating a labeled sample to a function $H\in\mathcal{F}$ (just as in concept class learning).

\begin{defin}[\cite{vid}]
A function class $\mathcal{F}$ is {\em distribution-free Probably Approximately Correct learnable} if there exists an algorithm $L:\cup_{m\in\mathbb{N}}(X\times\unit)^m\to\mathcal{F}$ with the following property: for every $\epsilon>0$, for every $\delta>0$, there exists a $M\in\mathbb{N}$ such that for every $f\in\mathcal{F}$, for every probability measure $\mu$, for every $m\geq M$, for any $x_1,\ldots,x_m\in X$, we have $E_\mu(H_m,f)<\epsilon$ with confidence at least $1-\delta$, where $H_m=L((x_1,f(x_1)),\ldots,(x_m,f(x_m)))$.
\end{defin}

Both definitions of PAC learning contain the $\epsilon$ and $\delta$ parameters. The error parameter $\epsilon$ is used because the hypothesis is not required to have zero error - only an arbitrarily small error. The risk parameter $\delta$ exists because there is no guarantee that any collection of sufficiently large training points leads to a valid hypothesis; the learning algorithm is only expected to produce a valid hypothesis with the sample points with confidence at least $1-\delta$. Hence, the name ``Probably ($\delta$) Approximately ($\epsilon$) Correct" is used \cite{kearns}.

The following example illustrates that the set of all axis-aligned rectangles in $\mathbb{R}^2$ is distribution-free PAC learnable. Both the statement and its proof can be found in Chapter 3 of \cite{vid} and Chapter 1 of \cite{kearns}.

\begin{example}
In $X = \mathbb{R}^2$, the concept class $\mathcal{C} = \{\left[a,b\right]\times\left[c,d\right]:a,b,c,d\in\mathbb{R} \}$ is distribution-free PAC learnable.
\end{example}

\begin{proof}
Let $\epsilon,\delta>0$. Given a concept $A$ and any sample of $m$ training points $x_1,\ldots,x_m\in X$, define the hypothesis concept $H_m$ to be the intersection of all rectangles containing only training points $x_i$ such that $\chi_A(x_i) = 1$. In other words, $H_m$ is the smallest rectangle that contains only the sample points {\em in} $A$.

Let $\mu$ be any probability measure, and in fact, $H_m \bigtriangleup A = A \setminus H_m$, which can be broken down into four sections $T_1,\ldots,T_4$. If we can conclude that
\[
\mu\left(\bigcup_{i=1}^4 T_i\right)<\epsilon,
\]
with confidence at least $1-\delta$, then the proof is complete. 

Consider the top section $T_1$ and define ${\tilde T_1}$ to be the rectangle along the top parts of $A$ whose measure is exactly $\epsilon/4$. The event ${\tilde T_1}\subseteq T_1$, which is equivalent to $\mu(T_1)\geq \epsilon/4$, holds exactly when no points in the sample $x_1,\ldots,x_m$ fall in ${\tilde T_1}$, and the probability of this event (which is the measure of all such $m$-tuples of $(x_1,\ldots,x_m)\in X^m$ where $x_i\notin {\tilde T_1}$ for all $i=1,\ldots,m$) is
\[
\left(1-\frac{\epsilon}{4}\right)^m.
\]
Similarly, the same holds for the other three sections $T_2,\ldots,T_4$. Therefore, the probability that there exists at least one $T_i$ such that $\mu(T_i)\geq \epsilon/4$, where $i\in\{1,\ldots,4\}$, is at most
\[
4\left(1-\frac{\epsilon}{4}\right)^m.
\]
Hence, as long as we pick $m$ large enough that $4(1-\epsilon/4)^m\leq \delta$, with confidence (probability) at least $1-\delta$, $\mu(T_i)<\epsilon/4$ for every $i=1,\ldots,4$ and thus,
\[
\mu(H_m \bigtriangleup A) = \mu\left(\bigcup_{i=1}^4T_i\right)\leq \mu(T_1)+\ldots +\mu(T_4)<4\left(\frac{\epsilon}{4}\right) = \epsilon.
\]
Please note that this argument, though very intuitive, actually requires the classical Glivenko-Cantelli theorem.

In summary, as long as $m\geq (4/\epsilon) \ln(4/\delta)$, with confidence at least $1-\delta$, $\mu(H_m\bigtriangleup A)<\epsilon$. We note that this estimate of the sample size only depends on $\epsilon$ and $\delta$, so $\mathcal{C}$ is indeed distribution-free PAC learnable.
\end{proof}

In the next section, a fundamental theorem which characterizes concept class distribution-free PAC learning will be stated, and two more concept classes, one learnable and the other not,\footnote{They are direct results of the theorem.} will be given. However, in order to state this theorem, the notion of shattering, which is essential in learning theory, must be introduced.

\newpage
\section{The Vapnik-Chervonenkis Dimension} The Vapnik-Chervonenkis dimension is a combinatorial parameter which is defined using the notion of shattering, developed first in 1971 by Vapnik and Chervonenkis.

\begin{defin}[\cite{vcdefinition}]
Given any set $X$ and a collection $\mathcal{A}$ of subsets of $X$, the collection $\mathcal{A}$ {\em shatters} a subset $S\subseteq X$ if for every $B\subseteq S$, there exists $A\in\mathcal{A}$ such that
\[
A \cap S = B.
\]
\end{defin}

There is an equivalent condition, which is sometimes easier to work with, to shattering, expressed in terms of characteristic functions of subsets of $X$.

\begin{prop}\label{equivalentshattering}
The collection $\mathcal{A}$ shatters a subset $S=\{x_1,\ldots,x_n\}\subseteq X$ if and only if for every $e=(e_1,\ldots,e_n)\in\{0,1\}^n$, there exists $A\in\mathcal{A}$ such that
\[
\chi_A(x_i) = e_i,
\]
for all $i=1,\ldots, n$.
\end{prop}
\begin{proof}
Trivial.
\end{proof}

\begin{defin}[\cite{vcdefinition}]
The {\em Vapnik-Chervonenkis (VC) dimension} of the collection $\mathcal{A}$, denoted by $\mathrm{VC}(\mathcal{A})$, is defined to be the cardinality of the largest finite subset $S\subseteq X$ shattered by $\mathcal{A}$. If $\mathcal{A}$ shatters arbitrarily large finite subsets of $X$, then the VC dimension of $\mathcal{A}$ is defined to be $\infty$.
\end{defin}

The VC dimension is defined for every collection $\mathcal{A}$ of subsets of any set $X$, so in particular, $X=(X,\mathcal{S})$ can be a measurable space and $\mathcal{A} = \mathcal{C}$ can be a concept class. 

The following are a few examples of how to calculate VC dimensions in the context of $X = \mathbb{R}^n$. In order to prove the VC dimension of a concept class $\mathcal{C}$ is $d$, we must provide a subset $S\subseteq X$ with cardinality $d$ which is shattered by $\mathcal{C}$ and prove that no subset with cardinality $d+1$ can be shattered by $\mathcal{C}$.

\begin{example}\label{powerset}
If $X = \mathbb{R}$, then the powerset of $X$ has infinite VC dimension. More generally, for every infinite set $X$, $\mathrm{VC}(\mathcal{P}(X))=\infty$.
\end{example}

\begin{example}
In the space $X =\mathbb{R}$, let $\mathcal{C} = \{\left[a,b\right]: a,b\in \mathbb{R},a<b\}$ be the collection of all closed intervals. Then, $\mathrm{VC}(\mathcal{C}) = 2$. 
\end{example}

\begin{proof}
Consider the subset $S=\{1,2\}\subseteq \mathbb{R}$; $\mathcal{C}$ shatters $S$ because
\[
\left[a,b\right]\cap S = \begin{cases}
\emptyset& \quad\textnormal{ if }a> 2 \textnormal{ or } b<1\\
\{1\}& \quad\textnormal{ if }a\leq 1,b<2\\
\{2\}&\quad\textnormal{ if }a>1,b\geq 2\\
\{1,2\}&\quad\textnormal{ if }a\leq 1,b\geq 2.
\end{cases}
\]

On the other hand, given any subset $S=\{x,y,z\}\subseteq \mathbb{R}$ with three distinct points, and assume the order to be $x<y<z$. Then, there are no closed interval in $\mathcal{C}$ containing $x$ and $z$ but not $y$.
\end{proof}

\begin{example}\label{newvcexample}
Consider the space $X = \mathbb{R}^n$. A hyperplane $H_{\vec{a},b}$ is defined by a nonzero vector $\vec{a}=(a_1,\ldots,a_n)\in \mathbb{R}^n$ and a scalar $b\in\mathbb{R}$:
\begin{align*}
H_{\vec{a},b}&=\{\vec{x}=(x_1,\ldots,x_n)\in\mathbb{R}^n:\vec{x}\cdot\vec{a}=b\}\\
&=\{\vec{x}=(x_1,\ldots,x_n)\in\mathbb{R}^n:x_1a_1+\ldots+x_na_n =b\}.
\end{align*}
Write $\mathcal{C}$ as the set of all hyperplanes: $\mathcal{C}=\{H_{\vec{a},b}:\vec{a}\in\mathbb{R}^n\setminus\{\vec{0}\},b\in\mathbb{R}\}$. Then $\mathrm{VC}(\mathcal{C}) = n$.

\begin{proof}
Consider the subset $S=\{{\vec e}_1,\ldots,{\vec e}_n\}\subseteq \mathbb{R}^n$, where ${\vec e}_i$ is the vector with $1$ on the $i$-th component and $0$ everywhere else. Suppose $B\subseteq S$ and there are two cases to consider:
\begin{enumerate}
\item If $B = \emptyset$, then let ${\vec a} = (1,1,\ldots,1)\in\mathbb{R}^n$ and the hyperplane $H_{\vec{a},-1} = \{\vec{x}=(x_1,\ldots,x_n)\in\mathbb{R}^n:x_1+\ldots+x_n =-1\}$ is disjoint from $S$.
\item If $B \neq \emptyset$, then set $\vec{a }= (a_1,\ldots,a_n)\in\mathbb{R}^n\setminus\{{\vec 0}\}$, where $a_i = \chi_{B}({\vec e}_i)$. Then the hyperplane $H_{\vec{a},1} = \{\vec{x}=(x_1,\ldots,x_n)\in\mathbb{R}^n:x_1a_1+\ldots+x_na_n =1\}$ satisfies
\[
H_{\vec{a},1} \cap S = B.
\]
\end{enumerate}

Moreover, no subset $S=\{{\vec x}_1,\ldots,{\vec x}_n,{\vec x}_{n+1}\}\subseteq \mathbb{R}^n$ with cardinality $n+1$ can be shattered by $\mathcal{C}$. At best, there exists a unique hyperplane $H_{\vec{a},b}$ containing $n$ of these points, say $\{{\vec x}_1,\ldots,{\vec x}_n\}$, so if ${\vec x}_{n+1} \in H_{\vec{a},b}$, then there are no hyperplanes that include ${\vec x}_1,\ldots,{\vec x}_n$, but not ${\vec x}_{n+1}$. Otherwise, if ${\vec x}_{n+1} \notin H_{\vec{a},b}$, then there are no hyperplanes that include ${\vec x}_1,\ldots,{\vec x}_n,{\vec x}_{n+1}$.
\end{proof}

\end{example}

The first example is trivial and the second is fairly well-known, seen in \cite{kearns} and \cite{vcinterval}, but we believe the third, Example \ref{newvcexample}, is a new result.

A very important concept related to shattering is the growth of all the possible subsets $A\cap S$, for $A\in\mathcal{C}$, as $S\subseteq X$ increases in size. It is clear that this growth is always exponential if $\mathcal{C}$ has infinite VC dimension; Sauer's Lemma explains the growth when $\mathrm{VC}(\mathcal{C})<\infty$.

\subsection{Sauer's Lemma}

Given a concept class $\mathcal{C}$ of $X$, another way to express that $\mathcal{C}$ shatters a subset $S\subseteq X$, with cardinality $n$, is to consider the set of all $A\cap S$, where $A\in\mathcal{C}$. Following Chapter 4 of \cite{vid}, $\mathcal{C}$ shatters $S$ if and only if
\[
|\{A \cap S: A\in\mathcal{C}\}| = 2^n.
\]
More generally, for any subset $S\subseteq X$, define
\[
\pi(S;\mathcal{C}) = |\{A \cap S: A\in\mathcal{C}\}|
\] 
and 
\[
\pi(n;\mathcal{C}) = \max_{|S|= n}\pi(S;\mathcal{C}).
\]
Then, the VC dimension of $\mathcal{C}$ can now be expressed in terms of the growth of $\pi(n;\mathcal{C})$ as $n$ gets large.

\begin{prop} \label{pidefinition}
Given a concept class $\mathcal{C}$, the following conditions are equivalent:
\begin{enumerate}
\item $\mathrm{VC}(\mathcal{C}) \geq n$;
\item $\mathcal{C}$ shatters some subset $S\subseteq X$ with cardinality $n$;
\item $\pi (n;\mathcal{C}) = 2^n$.
\end{enumerate}
Moreover, the class $\mathcal{C}$ has infinite VC dimension if and only if $\pi(n;\mathcal{C}) = 2^n$ for all $n\in\mathbb{N}$. Conversely, $\mathcal{C}$ has finite VC dimension, say $\mathrm{VC}(\mathcal{C}) \leq d$, if and only if $\pi (n;\mathcal{C}) <2 ^n$ for all $n> d$.
\end{prop}
\begin{proof}
The proof follows from the fact that $\mathcal{C}$ shatters $S$ if and only if $\pi(S;\mathcal{C})=2^n$.
\end{proof}

The extremely interesting fact, as seen in the next theorem, is that if $\mathcal{C}$ has finite VC dimension $d$, then $\pi(n;\mathcal{C})$ is bounded by a polynomial in $n$ of degree $d$, for $n\geq d$. This result, called Sauer's Lemma, was first proven in 1972 by Sauer. In other words, as $n$ gets large, $\pi(n;\mathcal{C})$ is either always an exponential function with base $2$ or eventually bounded by a polynomial function of a fixed degree.

\begin{them}[Sauer's Lemma \cite{sauerlemma}]
Suppose a concept class $\mathcal{C}$ has finite VC dimension $d$. Then
\[
\pi (n;\mathcal{C}) \leq \left(\frac{en}{d}\right)^d,
\]
for all $n\geq d\geq 1$.
\end{them}

Of course, everything in this subsection, including Sauer's Lemma, is true for any collection of subsets of any set but in the context of statistical learning theory, Sauer's Lemma is particularly useful because it is used to prove the equivalence of a concept class having finite VC dimension and the class being distribution-free PAC learnable.

\subsection{Characterization of concept class distribution-free PAC learning}

The following is one of the main theorems concerning PAC learning, whose proof results from Vapnik and Chervonenkis' paper \cite{vcdefinition} in 1971 and the 1989 paper \cite{vctheorem} by Blumer et al..

\begin{them}[\cite{vcdefinition} and \cite{vctheorem}]
\label{vctheorem}
Let $\mathcal{C}$ be a concept class of a measurable space $(X,\mathcal{S})$. The following are equivalent:
\begin{enumerate}
\item $\mathcal{C}$ is distribution-free Probably Approximately Correct learnable.
\item $\mathrm{VC}(\mathcal{C})<\infty$.
\end{enumerate}
\end{them}

Both directions of the proof require expressing the number of sample training points required for learning in terms of the VC dimension of $\mathcal{C}$; Sauer's Lemma is used to provide a sufficient number of points required for learning in the direction $2) \Rightarrow 1)$.

Using Theorem \ref{vctheorem}, one can more easily determine whether a given concept class is distribution-free PAC learnable.

\begin{example}
Let $X$ be any infinite set. Then the powerset $\mathcal{P}(X)$ is not distribution-free PAC learnable.
\end{example}

\begin{example}
The set of all hyperplanes $\mathcal{C}=\{H_{\vec{a},b}:\vec{a}\in\mathbb{R}^n\setminus\{\vec{0}\},b\in\mathbb{R}\}$, as defined in Example \ref{newvcexample}, is distribution-free PAC learnable.
\end{example}

Both examples come directly from the calculations of their concept classes' VC dimensions in Examples \ref{powerset} and \ref{newvcexample} and from Theorem \ref{vctheorem}.

Every concept class $\mathcal{C}$ can be viewed as a function class $\mathcal{F}_\mathcal{C}=\{\chi_A:X\to\unit:A\in\mathcal{C}\}$, as seen in Section \ref{pacsection}, so a natural question is whether the notion of shattering can be generalized. Indeed, the next section introduces the Fat Shattering dimension of scale $\epsilon$, which is a generalization of the VC dimension.

\newpage
\section{The Fat Shattering Dimension}
Let $\epsilon>0$ from this section onwards. A combinatorial parameter which generalizes the Vapnik-Chervonenkis dimension is the Fat Shattering dimension of scale $\epsilon$, defined first by Kearns and Schapire in 1994. 

This dimension, assigned to function classes, involves the notion of {\em $\epsilon$-shattering}, but similar to the notion of (regular) shattering, it can be defined for any collection of functions $f:X\to\unit$, where $X$ is any set, but for sake of this report, the following sections (still) assume $X=(X,\mathcal{S})$ is a measurable space and the collection of functions is a function class $\mathcal{F}$.

\begin{defin}[\cite{fatdefinition}]
Let $\mathcal{F}$ be a function class. Given a subset $S=\{x_1,\ldots,x_n\}\subseteq X$, the class $\mathcal{F}$ {\em $\epsilon$-shatters} $S$, with {\em witness} $c=(c_1,\ldots,c_n)\in \left[0,1\right]^n$, if for every $e\in \{0,1\}^n$, there exists $f\in \mathcal{F}$ such that
\[
f(x_i)\geq c_i + \epsilon \textnormal{ for }e_i=1,\textnormal{ and }f(x_i)\leq c_i-\epsilon \textnormal{ for }e_i=0.
\]
\end{defin}

\begin{defin}[\cite{fatdefinition}]
The {\em Fat Shattering dimension of scale $\epsilon>0$} of $\mathcal{F}$, denoted by $\mathrm{fat}_\epsilon(\mathcal{F})$, is defined to be the cardinality of the largest finite subset of $X$ that can be $\epsilon$-shattered by $\mathcal{F}$. If $\mathcal{F}$ can $\epsilon$-shatter arbitrarily large finite subsets, then the Fat Shattering dimension of scale $\epsilon$ of $\mathcal{F}$ is defined to be $\infty$.
\end{defin}

When the function class $\mathcal{F}$ consists of only functions taking values in $\{0,1\}$, then the Fat Shattering dimension of any scale $\epsilon\leq 1/2$ of $\mathcal{F}$ agrees with the VC dimension of the corresponding collection of subsets of $X$, induced by the (indicator) functions in $\mathcal{F}$.

\begin{prop}
Suppose a function class $\mathcal{F}$ consists of only binary functions $f:X\to\{0,1\}$. For every $f\in \mathcal{F}$, there exists a unique subset $A_f\subseteq X$ such that $\chi_{A_f} = f$. Moreover, write $\mathcal{C} = \{A_f:f\in\mathcal{F}\}$ and $\mathrm{VC}(\mathcal{C}) = \mathrm{fat}_\epsilon(\mathcal{F})$ for all $\epsilon\leq 0.5$.
\end{prop}

\begin{proof}
The first statement, of the existence of a unique subset $A_f\subseteq X$ for every binary function $f$, is clear. Let $\epsilon\leq 0.5$. To show that $\mathrm{VC}(\mathcal{C}) = \mathrm{fat}_\epsilon(\mathcal{F})$, it suffices to prove that $\mathcal{C}$ shatters $S = \{x_1,\ldots,x_n\}$ if and only if $\mathcal{F}$ $\epsilon$-shatters $S$. 

The equivalent condition to shattering as seen in Proposition \ref{equivalentshattering} will be used. Suppose $\mathcal{C}$ shatters $S$ and define $c=(0.5,0.5,\ldots,0.5)\in\unit^n$. For every $e\in\{0,1\}^n$, there exists $A_f\in \mathcal{C}$, where $f\in \mathcal{F}$, such that
\[
\chi_{A_f}(x_i) = e_i,
\]
for all $i=1,\ldots, n$ and thus,
\[
f(x_i) = \chi_{A_f}(x_i) =e_i \geq0.5 + \epsilon \textnormal{ for }e_i=1
\]
and
\[
f(x_i)=\chi_{A_f}(x_i) = e_i\leq 0.5 -\epsilon \textnormal{ for }e_i=0.
\]

Conversely, suppose $\mathcal{F}$ $\epsilon$-shatters $S$, with witness $c = (c_1,\ldots,c_n)\in \unit^n$. Let $e\in \{0,1\}^n$ and there exists $f\in\mathcal{F}$ such that
\[
f(x_i)\geq c_i + \epsilon \textnormal{ for }e_i=1,\textnormal{ and }f(x_i)\leq c_i-\epsilon \textnormal{ for }e_i=0,
\]
but $f$ is binary and $\epsilon$ is strictly positive, so $f(x_i)\geq c_i + \epsilon$ implies $f(x_i) = 1$ for $e_i = 1$ and $f(x_i)\leq c_i - \epsilon$ implies $f(x_i) = 0$ for $e_i = 0$. As a result, consider $A_f\in\mathcal{C}$ and
\[
\chi_{A_f}(x_i) = f(x_i) = e_i
\]
for all $i=1,\ldots,n$. Therefore, $\mathrm{VC}(\mathcal{C}) = \mathrm{fat}_\epsilon(\mathcal{F})$.
\end{proof}

Here is an example of a commonly used function class which we proved, independent of any sources, to have infinite Fat Shattering dimension of scale $\epsilon$.

\begin{example}
Let $X = \mathbb{R}^+$ and let $\mathcal{F}$ be the set of all continuous functions $f:X\to\unit$. Then $\mathrm{fat}_\epsilon(\mathcal{F}) = \infty$ for all $0<\epsilon\leq 0.5$.
\end{example}
\begin{proof}
Suppose $0<\epsilon\leq 0.5$, and consider a collection of continuous $\unit$-valued functions defined as follows. Given $e\in\{0,1\}^\mathbb{N}$, a countable binary sequence, define $f_e:X\to\unit$ by
\[
f_e(x)=\begin{cases}
1&\quad\textnormal{ if }e_i=1\\
0&\quad\textnormal{ if }e_i=0,
\end{cases}
\]
if $x = i\in\mathbb{N}$. Otherwise, for $x\in\left[m,m+1\right]$, with $m\in\mathbb{N}$,
\[
f_e(x)=\begin{cases}
-(x-m) + 1&\quad\textnormal{ if }e_m = 1,e_{m+1} = 0\\
(x-m)&\quad\textnormal{ if }e_m = 0, e_{m+1} = 1\\
e_m&\quad\textnormal{ if } e_m = e_{m+1}.
\end{cases}
\]
For each $e\in\{0,1\}^\mathbb{N}$, $f_e$ is continuous because it is defined as a step function of lines which agree on the overlaps. Write $F = \{f_e:e\in\{0,1\}^\mathbb{N}\}$ and $F\subseteq \mathcal{F}$. To show that $\mathrm{fat}_\epsilon(\mathcal{F}) = \infty$, it suffices to prove that $\mathrm{fat}_\epsilon (F) = \infty$. Consider the subset $S=\{1,\ldots,n\}\subseteq X$ for any $n\in\mathbb{N}$, and the collection $F$ $\epsilon$-shatters $S$ with witness $c = (0.5,0.5,\ldots,0.5) \in \unit^n$: for each $e\in\{0,1\}^n$, it can be extended to a countable binary sequence ${\tilde e}$, where ${\tilde e_i} = e_i$ for all $i = 1,\ldots, n$ and ${\tilde e_i} = 0$ otherwise. Then, it is clear that
\[
f_{{\tilde e}}(x_i) = 1\geq c_i + \epsilon \textnormal{ for }{\tilde e_i}=1,\textnormal{ and }f(x_i) = 0\leq c_i-\epsilon \textnormal{ for }{\tilde e_i}=0,
\]
with $x_i = i\in S$ for $i = 1,\ldots, n$.
\end{proof}

With the generalization from a concept class to a function class, a natural question is whether the finiteness of the Fat Shattering dimension of all scales $\epsilon$ for a function class $\mathcal{F}$ is equivalent to $\mathcal{F}$ being distribution-free PAC learnable. This question is addressed in the following subsection.

\subsection{Sufficient condition for function class distribution-free PAC learning}

One direction of Theorem \ref{vctheorem} can be generalized and stated in terms of the Fat Shattering dimension of scale $\epsilon$ of a function class.

\begin{them}[\cite{fatshattering} and \cite{vid}]
\label{sufficientfatlearning}
Let $\mathcal{F}$ be a function class. If $\mathrm{fat}_\epsilon(\mathcal{F})<\infty$ for all $\epsilon>0$, then $\mathcal{F}$ is distribution-free PAC learnable.
\end{them}

However, the converse to Theorem \ref{sufficientfatlearning} is false. There exists a distribution-free PAC learnable function class with infinite Fat Shattering dimension of some scale $\epsilon$. 

In fact, for every concept class $\mathcal{C}$ with cardinality $\aleph_0$ or $2^{\aleph_0}$, there is an associated function class $\mathcal{F}_\mathcal{C}$ defined as follows. Set up a bijection $b:\mathcal{C} \to \left[0,1/3\right]$ or to $\left[0,1/3\right]\cap \mathbb{Q}$, depending on the cardinality of $\mathcal{C}$, and for every $A\in\mathcal{C}$, define a function $f_A:X\to\unit$ by
\[
f_A(x) = \chi_A(x)+(-1)^{\chi_A(x)}b(A).
\]
Now, write $\mathcal{F}_\mathcal{C}=\{f_A:A\in\mathcal{C}\}$. Note that $\mathcal{F}_\mathcal{C}$ can be thought of the collection of all indicator functions of $A\in\mathcal{C}$, except that each ``indicator" function $f_A$ has two unique identifying points $b(A)$ and $1-b(A)$, instead of simply $0$ and $1$. The following proposition provides many counterexamples to Theorem \ref{sufficientfatlearning}, which are much simpler than the one found in \cite{vid}.

The construction of the function class $\mathcal{F}_\mathcal{C}$ and the proposition below are developed from an idea of Example 2.10 in \cite{pestovnote}.

\begin{prop}\label{functionconverse}
Let $\mathcal{C}$ be a concept class. The associated function class $\mathcal{F}_\mathcal{C}=\{f_A:A\in\mathcal{C}\}$, defined in the previous paragraph, is always distribution-free PAC learnable; this class has infinite Fat Shattering dimension of all scales $\epsilon<1/6$ if $\mathcal{C}$ has infinite VC dimension.
\end{prop}
\begin{proof}
The function class $\mathcal{F}_\mathcal{C}$ is distribution-free PAC learnable because every function $f_A\in\mathcal{F}_\mathcal{C}$ can be uniquely identified with just one point $x_0\in X$ in any labeled sample: $f_A(x_0) \in \{b(A), 1 - b(A)\}$ uniquely determines $A$ and thus, $f_A$.

Furthermore, suppose $\mathcal{C}$ has infinite VC dimension. Let $n\in\mathbb{N}$ be arbitrary and because $\mathrm{VC}(\mathcal{C})=\infty$, there exists $S=\{x_1,\ldots,x_n\}$ such that $\mathcal{C}$ shatters $S$. Suppose $\epsilon<1/6$ and we claim that $\mathcal{F}_\mathcal{C}$ $\epsilon$-shatters $S$ with witness $c=(0.5,\ldots,0.5)\in\unit^n$. Indeed, let $e\in\{0,1\}^n$ and there exists $A\in\mathcal{C}$ such that
\[
\chi_A(x_i) = e_i,
\]
for all $i=1,\ldots,n$, by Proposition \ref{equivalentshattering}. As a result,
\[
f_A(x_i) = 1 - b(A) \geq 0.5 + \epsilon\textnormal{ for } e_i = 1
\]
and
\[
f_A(x_i) = b(A) \leq 0.5 - \epsilon\textnormal{ for } e_i = 0.
\]

Consequently, $\mathcal{F}_\mathcal{C}$ has infinite Fat Shattering dimension of all scales $\epsilon<1/6$.
\end{proof}

The next section explains the main result of our research: bounding the Fat Shattering dimension of scale $\epsilon$ of a composition function class which is built with a continuous logic connective.
\newpage
\section{The Fat Shattering Dimension of a Composition Function Class}

The goals of this section are to construct a new function class from old ones by means of a continuous logic connective and to bound the Fat Shattering dimension of scale $\epsilon$ of the new function class in terms of the dimensions of the old ones. The following subsection provides this construction, which can be found in Chapter 4 of \cite{vid}, in the context of concept classes using a connective of classical logic.

\subsection{Construction in the context of concept classes}\label{vcu}

Let $\mathcal{C}_1, \mathcal{C}_2,\ldots,\mathcal{C}_k$ be concept classes, where $k\geq 2$, and let $u:\{0,1\}^k\to\{0,1\}$ be any function, commonly known as a connective of classical logic. A new collection of subsets of $X$ arises from $\mathcal{C}_1,\ldots,\mathcal{C}_k$ as follows. 

As mentioned earlier in this report, every element $A\in \mathcal{C}_i$ can be identified as a binary function $f:X\to\{0,1\}$, namely its characteristic function $f = \chi_A$, and vice versa. Now, for any $k$ functions $f_1,\ldots,f_k:X\to\{0,1\}$, where $f_i\in\mathcal{C}_i$ with $i=1,\ldots,k$, consider a new function $u(f_1,\ldots,f_k):X\to\{0,1\}$ defined by
\[
u(f_1,\ldots,f_k)(x) = u(f_1(x),\ldots,f_k(x)).
\]
The set of all possible $u(f_1,\ldots,f_k)$, denoted by $u(\mathcal{C}_1,\ldots,\mathcal{C}_k)$, is given by
\[
u(\mathcal{C}_1,\ldots,\mathcal{C}_k) = \{u(f_1,\ldots,f_k) : f_i\in\mathcal{C}_i\}.
\]

For instance, when $k=2$, we can consider the ``Exclusive Or" connective $\oplus:\{0,1\}^2\to\{0,1\}$ defined by
\[
p\oplus q = (p\wedge \neg q) \vee (\neg p \wedge q),
\]
which corresponds to the symmetric difference operation. Then, our new concept class constructed from $\mathcal{C}_1$ and $\mathcal{C}_2$ is
\[
\{A_1\bigtriangleup A_2:A_1\in\mathcal{C}_1,A_2\in\mathcal{C}_2\}.
\]

The next theorem states that if $\mathcal{C}_1,\mathcal{C}_2,\ldots,\mathcal{C}_k$ all have finite VC dimension to start with, then regardless of $u$, the new collection $u(\mathcal{C}_1,\ldots,\mathcal{C}_k)$ always has finite VC dimension.
\begin{them}[\cite{vid}]
Let $k\geq 2$. Suppose $\mathcal{C}_1,\ldots,\mathcal{C}_k$ are concept classes, each viewed as a collection of binary functions, and $u:\{0,1\}^k\to \{0,1\}$ is any function. If the VC dimension of $\mathcal{C}_i$ is finite for all $i=1,\ldots, k$. Then there exists a constant $\alpha=\alpha_k$\footnote{More specifically, $\alpha = \alpha_k$ is the smallest integer such that
\[
k < \frac{\alpha}{\log(e\alpha)}.
\]}, which depends only on $k$, such that
\[
\mathrm{VC}(u(\mathcal{C}_1,\ldots,\mathcal{C}_k))< d\alpha_k,
\]
where $d =\displaystyle \max_{i=1}^k \mathrm{VC}(\mathcal{C}_i)$.
\end{them}

The proof of this theorem can be found in \cite{vid} and uses Sauer's Lemma to bound the VC dimension of $u(\mathcal{C}_1,\ldots,\mathcal{C}_k)$. The main objective of our project was to generalize this theorem for function classes, in terms of the Fat Shattering dimension of scale $\epsilon$, but the connective of classical logic $u$ would have to be replaced by a continuous logic connective, a continuous function $u:\unit^k\to\unit$.

\subsection{Construction of new function class with continuous logic connective}

In first-order logic, there are only two truth-values $0$ or $1$, so a connective is a function $\{0,1\}^k\to\{0,1\}$ in the classical sense. However, in continuous logic, truth-values can be found anywhere in the unit interval $\unit$. Therefore, we should consider a function $u:\unit^k\to\unit$, which will transform function classes, and require that $u$ be a continuous logic connective. In other words, $u$ should be continuous from the (product) metric space $\unit^k$ to the unit interval \cite{metricstructure}; in fact, because $u$ is continuous from a compact metric space to a metric space, it is automatically uniformly continuous.

The following provides the definition of a uniformly continuous function $u$ from any metric space to another, but we must first qualify $u$ with a modulus of uniform continuity.

\begin{defin}[See e.g. \cite{metricstructure}]
A {\em modulus of uniform continuity} is any function $\delta: \left(0,1\right]\to\left(0,1\right]$. 
\end{defin}

\begin{defin}[See e.g. \cite{metricstructure}]
\label{uniform}
Let $(M_1,d_1)$ and $(M_2,d_2)$ be two metric spaces. A function $u:M_1\to M_2$ is {\em uniformly continuous} if there exists (a modulus of uniform continuity) $\delta:\left(0,1\right]\to\left(0,1\right]$ such that for all $\epsilon\in\left(0,1\right]$ and $m_1,m_2\in M_1$, if $d_1(m_1,m_2)<\delta(\epsilon)$, then $d_2(u(m_1),u(m_2))<\epsilon$.

Such a $\delta$ is called a {\em modulus of uniform continuity for} $u$.
\end{defin}

In particular, $u:\unit^k\to\unit$, where $\unit^k$ is equipped with the $L_2$ product distance $d^2$, is uniformly continuous with modulus of uniform continuity $\delta$ if for every $\epsilon\in\left(0,1\right]$ and for every $(r_1,\ldots,r_k),(r_1',\ldots,r_k')\in\unit^k$,
\[
d^2 ((r_1,\ldots,r_k),(r_1',\ldots,r_k'))<\delta(\epsilon) \Rightarrow |u(r_1,\ldots,r_k) - u(r_1',\ldots,r_k')|<\epsilon.
\]

Given function classes $\mathcal{F}_1,\ldots,\mathcal{F}_k$ and a uniformly continuous function $u:\unit^k\to\unit$, consider the new function class $u(\mathcal{F}_1,\ldots,\mathcal{F}_k)$ defined by
\[
u(\mathcal{F}_1,\ldots,\mathcal{F}_k) = \{u(f_1,\ldots,f_k) : f_i\in\mathcal{F}_i\},
\]
where $u(f_1,\ldots,f_k)(x) =  u(f_1(x),\ldots,f_k(x))$ for all $x\in X$, just as in Section \ref{vcu} for concept classes, with $f_i \in \mathcal{F}_i$ and $i=1,\ldots, k$. Our main result states that the Fat Shattering dimension of scale $\epsilon$ of $u(\mathcal{F}_1,\ldots,\mathcal{F}_k)$ is bounded by a sum of the Fat Shattering dimensions of scale $\delta(\epsilon,k)$ of $\mathcal{F}_1,\ldots,\mathcal{F}_k$, where $\delta(\epsilon,k)$ is a function of the modulus of uniform continuity $\delta(\epsilon)$ for $u$ and $k$. It is a known result, seen in Chapter 5 of \cite{vid}, that this new class $u(\mathcal{F}_1,\ldots,\mathcal{F}_k)$ has finite Fat Shattering dimension of all scales $\epsilon>0$ (and thus, it is distribution-free PAC learnable) if each of $\mathcal{F}_1,\ldots,\mathcal{F}_k$ has finite Fat Shattering dimension of all scales, but no bounds were known.

\subsection{Main Result}

Fix $k\geq 2$ and the following theorem is our main new result.
\begin{them}\label{mainresult}
Let $\epsilon>0$, $\mathcal{F}_1,\ldots,\mathcal{F}_k$ be function classes of $X$, and $u:\unit^k\to\unit$ be a uniformly continuous function with modulus of continuity $\delta(\epsilon)$. Then
\[
\mathrm{fat}_{\epsilon}(u(\mathcal{F}_1,\ldots,\mathcal{F}_k) ) \leq \left(\frac{K\log(4c'k\sqrt{k}/(\delta(\epsilon/(2c'))\epsilon))}{K'\log(2)}\right)\sum_{i=1}^n \mathrm{fat}_{c\frac{\delta(\epsilon/(2c'))\epsilon}{k\sqrt{k}}}(\mathcal{F}_i),
\]
where $c,c', K, K'$ are some absolute constants.
\end{them}

Extracting the actual values of these absolute constants is not easy, and we hope to find them in future research. For this reason, comparing the bound in Theorem \ref{mainresult} with the existing estimate for the VC dimension of a composition concept class is difficult; however, in statistical learning theory, estimates for function class learning are generally much worse than estimates for concept class learning.

In order to prove Theorem \ref{mainresult}, for clarity, we first introduce an auxiliary function $\phi:\mathcal{F}_1\times\ldots\times\mathcal{F}_k\to \unit^X$, which is uniformly continuous from the metric space $\mathcal{F}_1\times\ldots\times\mathcal{F}_k$ with the $L_2$ product distance ${\tilde d^2}$ to the metric space $\unit^X$ with distance induced by the $L_2(\mu)$ norm, and prove the following lemma.
\begin{lemma}\label{auxphi}
Let $\epsilon>0$, $\mathcal{F}_1,\ldots,\mathcal{F}_k$ be function classes of $X$, and $\phi: \mathcal{F}_1\times\ldots\times\mathcal{F}_k\to \unit^X$ be uniformly continuous with some modulus of continuity $\delta(\epsilon,k)$, a function of $\epsilon$ and $k$. Then
\[
\mathrm{fat}_{c'\epsilon}(\phi(\mathcal{F}_1\times\ldots\times\mathcal{F}_k) ) \leq \left(\frac{K\log(2\sqrt{k}/\delta(\epsilon,k))}{K'\log(2)}\right)\sum_{i=1}^k \mathrm{fat}_{c\frac{\delta(\epsilon,k)}{\sqrt{k}}}(\mathcal{F}_i),
\]
where $c,c', K, K'$ are some absolute constants and the symbol $\phi(\mathcal{F}_1\times\ldots\times\mathcal{F}_k)$ simply represents the image of $\phi$.
\end{lemma}

Then, we will relate the two uniformly continuous functions $u$ and $\phi$.
\begin{lemma}\label{urelationphi}
Let $\epsilon>0$. If $u:\unit^k\to\unit$ is uniformly continuous with modulus of continuity $\delta(\epsilon)$, then the function $\phi:\mathcal{F}_1\times\ldots\times\mathcal{F}_k\to\unit^X$ defined by
\[
\phi(f_1,\ldots,f_k)(x) = u(f_1(x),\ldots,f_k(x))
\]
is also uniformly continuous with modulus of continuity $\frac{\delta(\epsilon/2)\epsilon}{2k}$, and in fact, $\phi(\mathcal{F}_1\times\ldots\times\mathcal{F}_k) = u(\mathcal{F}_1,\ldots,\mathcal{F}_k)$.
\end{lemma}

\subsection{Proofs}

In order to prove Lemma \ref{auxphi}, we first introduce the concept of an $\epsilon$-covering number for any metric space, based on \cite{entropy}, and relate this number for a function class to its Fat Shattering dimension of scale $\epsilon$ by using results from Mendelson and Vershynin \cite{entropy} and Talagrand \cite{talagrand}.

\begin{defin}
Let $\epsilon>0$ and suppose $(M,d)$ is a metric space. The {\em $\epsilon$-covering number}, denoted by $N(M,\epsilon, d)$, of $M$ is the minimal number $N$ such that there exists elements $m_1,m_2,\ldots,m_N\in M$ with the property that for all $m\in M$, there exists $i\in\{1,2,\ldots,N\}$ for which
\[
d(m,m_i)<\epsilon.
\]
The set $\{m_1,m_2,\ldots,m_N\}$ is called a {\em (minimal) $\epsilon$-net} of $M$.
\end{defin}

The following proposition relates the $\epsilon$-covering number of a product of metric spaces, with the $L_2$ product distance $d^2$, $M_1\times\ldots\times M_k$ to the $\frac{\epsilon}{\sqrt{k}}$-covering number of each space $M_i$.

\begin{prop}\label{productcovering}
Let $\epsilon>0$ and suppose $(M_1,d_1),\ldots,(M_k,d_k)$ are metric spaces, each with finite $\frac{\epsilon}{\sqrt{k}}$-covering numbers, $N_i = N(M_i,\frac{\epsilon}{\sqrt{k}},d_i)$ for $i = 1,\ldots, k$. Then
\[
N(M_1\times\ldots\times M_k,\epsilon, d^2)\leq \displaystyle\prod_{i= 1}^k N_i.
\]
\end{prop}
\begin{proof}
Let $C_i=\{a^i_1,\ldots,a^i_{N_i}\}$ be a minimal $\frac{\epsilon}{\sqrt{k}}$-net for $M_i$ with respect to distance $d_i$, where $i=1,\ldots,k$ and suppose $(a^1,\ldots,a^k)\in M_1\times\ldots\times M_k$. Then, for each $i=1,\ldots,k$, there exists $a^i_{j_i}\in C_i$, where $1\leq j_i\leq N_i$ such that $d_i(a^i,a^i_{j_i})<\frac{\epsilon}{\sqrt{k}}$. Hence,
\begin{align*}
d^2((a^1,\ldots,a^k),(a^1_{j_1},\ldots,a^k_{j_k})) & = \sqrt{\left((d_1(a^1,a^1_{j_1}))^2+\ldots+(d_k(a^k,a^k_{j_k}))^2\right)}\\
& <\sqrt{\left(\left(\frac{\epsilon}{\sqrt{k}}\right)^2 +\ldots +\left(\frac{\epsilon}{\sqrt{k}}\right)^2\right)}\\
& =\epsilon,
\end{align*}
where each $(a^1_{j_1},\ldots,a^k_{j_k})\in C_1\times\ldots\times C_k$, which has cardinality $\Pi_{i= 1}^k N_i$. Therefore, $N(M_1\times\ldots\times M_k,\epsilon, d^2) \leq \Pi_{i=1}^k N_i$.
\end{proof}

Also, if $u:M_1\to M_2$ is any uniformly continuous function with a modulus of uniform continuity $\delta(\epsilon)$ from any metric space to another, then the image of a minimal $\delta(\epsilon)$-net of $M_1$ under $u$ becomes an $\epsilon$-net for $u(M_1)$.

\begin{prop}\label{imageuniformcovering}
Let $\epsilon>0$ and suppose $(M_1,d_1)$ and $(M_2,d_2)$ are two metric spaces. If a function $u:M_1\to M_2$ is uniformly continuous with a modulus of continuity $\delta(\epsilon)$, then $N(u(M_1),\epsilon,d_2)\leq N(M_1,\delta(\epsilon),d_1)$, where $u(M_1)$ denotes the image of $u$.
\end{prop}

\begin{proof}
Suppose $N=N(M_1,\delta(\epsilon),d_1)$ is the $\delta(\epsilon)$-covering number for $M_1$ and let $\{m_1,\ldots,m_N\}$ be a $\delta(\epsilon)$-net for $M_1$. Hence for every $u(m)\in u(M_1)$, where $m\in M_1$, there exists $i\in\{1,\ldots,N\}$ such that
\[
d_1(m,m_i)<\delta(\epsilon),
\]
which implies $d_2(u(m),u(m_i))<\epsilon$ as $u$ is uniformly continuous. As a result, the set
\[
\{u(m_1),\ldots,u(m_N)\}
\]
is an $\epsilon$-net for $u(M_1)$, so
\[
N(u(M_1),\epsilon,d_2)\leq N(M_1,\delta(\epsilon),d_1).
\]
\end{proof}

In particular, we can view $\mathcal{F}_1,\ldots,\mathcal{F}_k$ as metric spaces, all with distances induced by the $L_2(\mu)$ norm and suppose $\phi:\mathcal{F}_1\times\ldots\times\mathcal{F}_k\to\unit^X$ is uniformly continuous with modulus of continuity $\delta(\epsilon,k)$. Then, by Proposition \ref{productcovering}, if $\mathcal{F}_1,\ldots,\mathcal{F}_k$ all have finite $\frac{\delta(\epsilon,k)}{\sqrt{k}}$-covering numbers, the metric space $\mathcal{F}_1\times\ldots\times\mathcal{F}_k$, with the $L_2$ product metric ${\tilde d^2}$, also has a finite $\delta(\epsilon,k)$-covering number: if we write $N(\mathcal{F}_i,\frac{\delta(\epsilon,k)}{\sqrt{k}},L_2(\mu))$ as the $\frac{\delta(\epsilon,k)}{\sqrt{k}}$-covering number for $\mathcal{F}_i$, then,
\[
N(\mathcal{F}_1\times\ldots\times\mathcal{F}_k,\delta(\epsilon,k),{\tilde d^2}) \leq \prod_{i=1}^k N(\mathcal{F}_i,\frac{\delta(\epsilon,k)}{\sqrt{k}},L_2(\mu)).
\]
Now, by Proposition \ref{imageuniformcovering}, 
\begin{align*}
N(\phi(\mathcal{F}_1\times\ldots\times\mathcal{F}_k),\epsilon,L_2(\mu))&\leq N(\mathcal{F}_1\times\ldots\times\mathcal{F}_k,\delta(\epsilon,k),{\tilde d^2}) \\
& \leq \prod_{i=1}^k N(\mathcal{F}_i,\frac{\delta(\epsilon,k)}{\sqrt{k}},L_2(\mu)).
\end{align*}

In other words, the $\epsilon$-covering number for $\phi(\mathcal{F}_1\times\ldots\times\mathcal{F}_k)$ is bounded by a product of the $\frac{\delta(\epsilon,k)}{\sqrt{k}}$-covering numbers of each $\mathcal{F}_i$. To prove Lemma \ref{auxphi}, we now state the main theorem of a paper written by Mendelson and Vershynin, which relates the $\epsilon$-covering number of a function class to its Fat Shattering dimension of scale $\epsilon$.

\begin{them}[\cite{entropy}]\label{mendelson}
Let $\epsilon>0$ and let $\mathcal{F}$ be a function class. Then for every probability measure $\mu$,
\[
N(\mathcal{F},\epsilon,L_2(\mu))\leq \left(\frac{2}{\epsilon}\right)^{K\mathrm{fat}_{c\epsilon}(\mathcal{F})}
\]
for absolute constants $c,K$.
\end{them}
And Talagrand provides the converse.
\begin{them}[\cite{talagrand}]\label{talagrand}
Following the notations of Theorem \ref{mendelson}, there exists a probability measure $\mu$ such that
\[
N(\mathcal{F},\epsilon,L_2(\mu))\geq 2^{K'\mathrm{fat}_{c'\epsilon}(\mathcal{F})},
\]
for absolute constants $c',K'$.
\end{them}
\begin{proof}[Proof of Lemma \ref{auxphi}]

By Propositions \ref{productcovering} and \ref{imageuniformcovering},
\[
N(\phi(\mathcal{F}_1\times\ldots\times\mathcal{F}_k),\epsilon,L_2(\mu))\leq \prod_{i=1}^k N(\mathcal{F}_i,\frac{\delta(\epsilon,k)}{\sqrt{k}},L_2(\mu)),
\]
so
\[
\log(N(\phi(\mathcal{F}_1\times\ldots\times\mathcal{F}_k),\epsilon,L_2(\mu))) \leq \sum_{i = 1}^k \log (N(\mathcal{F}_i,\frac{\delta(\epsilon,k)}{\sqrt{k}},L_2(\mu))).
\]
By Theorem \ref{mendelson},
\[
\log N(\mathcal{F}_i,\frac{\delta(\epsilon,k)}{\sqrt{k}},L_2(\mu))\leq K \mathrm{fat}_{c\frac{\delta(\epsilon,k)}{\sqrt{k}}}(\mathcal{F}_i)\log(2\sqrt{k}/\delta(\epsilon,k)),
\]
for any probability measure $\mu$ where $c,K$ are absolute constants. Moreover, by Theorem \ref{talagrand} for some probability measure $\mu$ and absolute constants $c',K'$,
\[
\log(N(\phi(\mathcal{F}_1\times\ldots\times\mathcal{F}_k),\epsilon,L_2(\mu))) \geq K'\mathrm{fat}_{c'\epsilon}(\phi(\mathcal{F}_1\times\ldots\times\mathcal{F}_k))\log(2)
\]
and altogether,
\begin{align*}
\mathrm{fat}_{c'\epsilon}(\phi(\mathcal{F}_1\times\ldots\times\mathcal{F}_k))&\leq\frac{\sum_{i=1}^k K\mathrm{fat}_{c\frac{\delta(\epsilon,k)}{\sqrt{k}}}(\mathcal{F}_i)\log(2\sqrt{k}/\delta(\epsilon,k)) }{K'\log(2)} \\
&= \left(\frac{K\log(2\sqrt{k}/\delta(\epsilon,k))}{K'\log(2)}\right)\sum_{i=1}^k \mathrm{fat}_{c\frac{\delta(\epsilon,k)}{\sqrt{k}}}(\mathcal{F}_i) .
\end{align*}
\end{proof}

Now, all that is left is to prove Lemma \ref{urelationphi}.
\begin{proof}[Proof of Lemma \ref{urelationphi}]
Suppose $u:\unit^k\to\unit$ is uniformly continuous with a modulus of continuity $\delta(\epsilon)$, where $\unit^k$ is a metric space with the $L_2$ product distance $d^2$. We claim that the function $\phi:\mathcal{F}_1\times\ldots\times\mathcal{F}_k\to\unit^X$ defined by
\[
\phi(f_1,\ldots,f_k)(x) = u(f_1(x),\ldots,f_k(x))
\]
is uniformly continuous with modulus of continuity $\frac{\delta(\epsilon/2)\epsilon}{2k}$. Let $\epsilon>0$ and
\[
(f_1,\ldots,f_k),(f_1',\ldots,f_k')\in\mathcal{F}_1\times\ldots\times\mathcal{F}_k. 
\]
Suppose
\begin{align*}
{\tilde d^2}((f_1,\ldots,f_k),(f_1',\ldots,f_k'))& = \sqrt{\left((||f_1 - f'_1||_2)^2 + \ldots + (||f_k - f'_k||_2)^2\right)}\\
&< \frac{\delta(\epsilon/2)\epsilon}{2k} = \sqrt{\frac{\delta(\epsilon/2)^2(\epsilon/2)^2}{k^2}}.
\end{align*}
Hence, for each $i=1,\ldots,k$,
\[
||f_i - f_i'||_2 = \sqrt{\left(\int_X (f_i(x) - f_i'(x))^2\,d\mu(x)\right)} < \sqrt{\frac{\delta(\epsilon/2)^2(\epsilon/2)^2}{k^2}}.
\]

Write $A_i = \{x\in X:|f_i(x) - f_i'(x)|\geq \sqrt{\frac{\delta(\epsilon/2)^2}{k}}\}$ and we must have that $\mu(A_i) <\frac{(\epsilon/2)^2}{k}$, for each $i=1,\ldots,k$. Otherwise,
\begin{align*}
\int_X (f_i(x) - f_i'(x))^2\,d\mu(x)  &= \int_{A_i} (f_i(x) - f_i'(x))^2\,d\mu(x) + \int_{X\setminus A_i} (f_i(x) - f_i'(x))^2\,d\mu(x)\\
&\geq \int_{A_i} \left(\sqrt{\frac{\delta(\epsilon/2)^2}{k}}\right)^2\,d\mu(x) + \int_{X\setminus A_i} (f_i(x) - f_i'(x))^2\,d\mu(x)\\
&= \mu(A_i) \left(\sqrt{\frac{\delta(\epsilon/2)^2}{k}}\right)^2+ \int_{X\setminus A_i} (f_i(x) - f_i'(x))^2\,d\mu(x)\\
&\geq \frac{(\epsilon/2)^2}{k}\frac{\delta(\epsilon/2)^2}{k} + \int_{X\setminus A_i} (f_i(x) - f_i'(x))^2\,d\mu(x)\\
&\geq \frac{\delta(\epsilon/2)^2(\epsilon/2)^2}{k^2},
\end{align*}
which is a contradiction. Now, write $A = A_1\cup\ldots\cup A_k$ and we have that $X\setminus A = \{x\in X: |f_i(x) - f_i'(x)|<\sqrt{\frac{\delta(\epsilon/2)^2}{k}},\textnormal{ for all }i=1,\ldots,k\}$. Suppose $x\in X\setminus A$ and then
\begin{align*}
d^2((f_1(x),\ldots,f_k(x)),(f_1'(x),\ldots,f_k'(x))) & = \sqrt{ |f_1(x) - f_1'(x)|^2 + \ldots + |f_k(x) - f_k'(x)|^2}\\
& <\sqrt{ \left(\frac{\delta(\epsilon/2)^2}{k} + \ldots +\frac{\delta(\epsilon/2)^2}{k}\right)}\\
&<\delta(\epsilon/2).
\end{align*}

Consequently, by the uniform continuity of $u$, for all $x\in X\setminus A$,
\[
|u(f_1(x),\ldots,f_k(x)) - u(f_1'(x),\ldots,f_k'(x))|<\epsilon/2.
\]

Finally,
\begin{align*}
\!\!\!\!\!\!\!\!\!\!\!\!||\phi(f_1,\ldots,f_k) - \phi(f_1',\ldots,f_k')||_2 &= \sqrt{ \left( \int_X (u(f_1(x),\ldots,f_k(x)) - u(f_1'(x),\ldots,f_k'(x)))^2\,d\mu(x)\right)}\\
& \leq \sqrt{\left( \int_{X\setminus A} (u(f_1(x),\ldots,f_k(x)) - u(f_1'(x),\ldots,f_k'(x)))^2\,d\mu(x)\right)} \\
&+\sqrt{\left( \int_A(u(f_1(x),\ldots,f_k(x)) - u(f_1'(x),\ldots,f_k'(x)))^2\,d\mu(x)\right)} \\
& < \sqrt{\left( \int_{X\setminus A} (\epsilon/2)^2\,d\mu(x)\right)} + \sqrt{\left( \int_A 1\,d\mu(x)\right)}\\
& \leq (\epsilon/2) + (\epsilon/2) = \epsilon,
\end{align*}
as $\mu(A) \leq \sum_{i=1}^k \mu(A_i) \leq k\left(\frac{(\epsilon/2)^2}{k}\right)=(\epsilon/2)^2$.
\end{proof}

Now we will prove our main theorem.
\begin{proof}[Proof of Theorem \ref{mainresult}]
By Lemma \ref{urelationphi}, if $u:\unit^k\to\unit$ is uniformly continuous with modulus of continuity $\delta(\epsilon)$, then $\phi:\mathcal{F}_1\times\ldots\times\mathcal{F}_k\to\unit^X$ defined by
\[
\phi(f_1,\ldots,f_k)(x) = u(f_1(x),\ldots,f_k(x))
\]
is also uniformly continuous with modulus of continuity $\frac{\delta(\epsilon/2)\epsilon}{2k}$. Then, apply Lemma \ref{auxphi} with $\delta(\epsilon,k) = \frac{\delta(\epsilon/2)\epsilon}{2k}$ and with a simple change of variables $c'\epsilon' \to \epsilon$, Theorem \ref{mainresult} follows directly.
\end{proof}

Altogether, we can summarize the maps in this section in the following two diagrams (where $i$ is the diagonal map):
\[
\xymatrix{
X\ar[r]^{i}& X^k \ar[rr]^{\!\!\!\!\!\!\!\! f_1\times\ldots\times f_k} &&\unit^k \ar[r]^{u}& \unit,
}
\]
while
\[
\xymatrix{
\mathcal{F}_1\times\ldots\times\mathcal{F}_k \ar[r]^{\,\,\,\,\,\,\,\,\,\,\,\,\,\phi} & \unit^X.
}
\]

This result is potentially useful because it allows us to construct new function classes using common continuous logic connectives and bound their Fat Shattering dimensions of scale $\epsilon$. For instance, the function $u:\unit^2\to\unit$ defined by $u(r_1,r_2) = r_1\cdot r_2$ (multiplication) is uniformly continuous with a modulus of continuity $\delta(\epsilon) = \frac{\epsilon}{2}$. Indeed, let $\epsilon>0$ and consider $(r_1,r_2),(r_1',r_2')\in \unit^2$. Suppose $d^2((r_1,r_2),(r_1',r_2'))<\delta(\epsilon)=\frac{\epsilon}{2}$, so
\[
|r_1 - r_1'| < \sqrt{ |r_1 - r_1'|^2 + |r_2 - r_2'|^2}<\frac{\epsilon}{2}
\]
and similarly, $|r_2 - r_2'|<\frac{\epsilon}{2}$. Then,
\begin{align*}
|u(r_1,r_2) - u(r_1',r_2')| & = |r_1r_2 - r_1'r_2'|\\
& = |r_1r_2 - r_1r_2' + r_1r_2' - r_1'r_2'|\\
& = |r_1(r_2-r_2') + r_2'(r_1-r_1')|\\
& \leq |r_1(r_2-r_2')| + |r_2'(r_1-r_1')|\\
& \leq |r_2-r_2'| + |r_1-r_1'|\\
& <\frac{\epsilon}{2} + \frac{\epsilon}{2} = \epsilon.
\end{align*}

As a result, if $\mathcal{F}_1$ and $\mathcal{F}_2$ are two function classes with finite Fat Shattering dimensions of some scale $\epsilon$, then the function class $u(\mathcal{F}_1,\mathcal{F}_2) = \mathcal{F}_1\mathcal{F}_2 = \{f_1\cdot f_2:f_1\in\mathcal{F}_1,f_2\in\mathcal{F}_2\}$, defined by point-wise multiplication, also has finite Fat Shattering dimension of scale $\epsilon$, up to some constant factor and Theorem \ref{mainresult} provides a precise bound.

We have made an interesting connection, which has not been explored much in the past, between continuous logic and PAC learning, and we plan to investigate this connection even further. For instance, the relationship of compositions of function classes and continuous logic may be interesting to study because compositions of uniformly continuous functions are again uniformly continuous.  Furthermore, we can try to add some topological structures to concept classes to see how PAC learning can be affected. The next section provides a couple of other possible future research topics.

\newpage
\section{Open Questions}

The definitions of distribution-free PAC learning, for both concept and function classes, in Section \ref{pacsection}, made no assumptions about probability measures, as a learning algorithm has to produce a valid hypothesis for any probability measure $\mu$. If we fix a probability measure $\mu$ and ask whether a concept class, or a function class, is PAC learnable, then we are working in the context of fixed distribution PAC learning.

\begin{defin}[\cite{vid}]\label{fixedpac}
Let $\mu$ be a probability measure. A function class $\mathcal{F}$ is {\em Probably Approximately Correct learnable under $\mu$} if there exists an algorithm $L:\cup_{m\in\mathbb{N}}(X\times\unit)^m\to\mathcal{F}$ with the following property: for every $\epsilon>0$, for every $\delta>0$, there exists a $M\in\mathbb{N}$ such that for every $f\in\mathcal{F}$, for every $m\geq M$, for any $x_1,\ldots,x_m\in X$, we have $E_\mu(H_m,f)<\epsilon$ with confidence at least $1-\delta$, where
\[
E_\mu(H_m,f) = \int_X |f(x) - g(x)|\, d\mu(x)
\]
and $H_m=L((x_1,f(x_1)),\ldots,(x_m,f(x_m)))$.
\end{defin}

When a function class $\mathcal{F}$ consists of only binary functions, i.e. $\mathcal{F}=\mathcal{C}$ is a concept class, there is a theorem, proved by Benedek and Itai in 1991, which gives a characterization of fixed distribution PAC learnability.

\begin{them}[\cite{fixedlearning}]
Fix a probability measure $\mu$ and consider a concept class $\mathcal{C}$. The following are equivalent:
\begin{enumerate}
\item $\mathcal{C}$ is Probably Approximately Correct learnable under $\mu$.
\item (Finite Metric Entropy condition) The $\epsilon$-covering number of $\mathcal{C}$ when viewed as a metric space with distance $d=\mu(\_\bigtriangleup\_)$ is finite for every $\epsilon>0$.
\end{enumerate}
\end{them}

However, there is no characterization for fixed distribution PAC learnability of a general function class. Talagrand had proved that a function class is a Glivenko-Cantelli (GC) function class with regard to a single measure $\mu$ if and only if the class has no witness of irregularity, a property that involves shattering \cite{witness1},\cite{witness2}. Every GC function class is PAC learnable under $\mu$ \cite{pestovnote}, but the property of having no witness of irregularity is strictly stronger than PAC learnability. We would like to propose the following conjecture for a possible characterization.

\begin{conj}
Fix a probability measure $\mu$ and consider a function class $\mathcal{F}$. Let $\epsilon>0$. The following are equivalent:
\begin{enumerate}
\item The function class $\mathcal{F}$ is PAC learnable under $\mu$ to accuracy $\epsilon$.\footnote{Being PAC learnable to accuracy $\epsilon$ means Definition \ref{fixedpac} is satisfied, but only for this particular $\epsilon$.}
\item There exists $M,N$ and $\gamma>0$ such that for all functions $f\in\mathcal{F}$, with probability at least $\gamma$, the set
$\{g\in{\mathcal F}: g_{|_{\bar x_N}} = f_{|_{\bar x_N}}\}$ has an $\epsilon$-covering number, with respect to the distance $d = E_\mu(\_,\_)$, of at most $M$, where ${\bar x_N}$ denotes a sample of $N$ points.
\end{enumerate}
\end{conj}

A very interesting research topic is to study this conjecture and either prove or disprove it. Also, by Proposition \ref{functionconverse}, the finiteness of the Fat Shattering dimension of all scales $\epsilon>0$ does not characterize function class PAC learning in the distribution-free case; consequently, another topic of research would be to come up with a new combinatorial parameter for a function class, related to the notion of shattering, which would characterize learning. This new parameter would have to solve the problem of unique identifications of functions, a problem that does not occur with concept classes.

Yet another possible research topic is to generalize the definitions of PAC learning and introduce observation noise, both in the fixed distribution and distribution-free cases. The paper \cite{noise} written by Bartlett et al. proves that the finiteness of the Fat Shattering dimension of all scales of a function class $\mathcal{F}$ is equivalent to $\mathcal{F}$ being distribution-free learnable under certain noise distributions. It would be interesting to generalize this result and/or apply it in the fixed distribution setting.
\newpage
\section{Conclusion}

This report introduces the definitions of Probably Approximately Correct learning for concept and function classes and defines the Vapnik-Chervonenkis dimension for concept classes and the Fat Shattering dimension of scale $\epsilon>0$ for function classes. Finiteness of the VC dimension characterizes concept class distribution-free PAC learning; however, the finiteness of the Fat Shattering dimension of all scales $\epsilon$ is still only sufficient for function class learning, and not necessary.

Given function classes $\mathcal{F}_1,\ldots,\mathcal{F}_k$, one can construct a new class $u(\mathcal{F}_1,\ldots,\mathcal{F}_k)$ using a continuous function $u:\unit^k\to\unit$, a continuous logic connective. The main new result of this report shows that the Fat Shattering dimension of scale $\epsilon$ of $u(\mathcal{F}_1,\ldots,\mathcal{F}_k)$ is bounded by a sum of the Fat Shattering dimensions of scale $\delta(\epsilon,k)$ of classes $\mathcal{F}_1,\ldots,\mathcal{F}_k$, up to some absolute constants. This result can be useful because it allows us to construct new function classes, which may be very natural objects, and bound their Fat Shattering dimensions.

\end{document}